\documentclass[preprint]{colt2022-arxiv} 

\usepackage{times}
\usepackage[T1]{fontenc}
\usepackage{lmodern}
\usepackage{microtype}
\usepackage{booktabs} 

\usepackage{hyperref}
\usepackage{thm-restate}

\usepackage{svg}

\usepackage{mathtools,  verbatim}
\usepackage{array}
\newcolumntype{R}{>{$}r<{$}} 
\newcolumntype{C}{>{$}c<{$}}

\usepackage{color}
\definecolor{darkblue}{rgb}{0.0,0.0,0.2}
\definecolor{darkgreen}{rgb}{0.0,0.3,0.0}
\hypersetup{colorlinks,breaklinks,
	linkcolor=darkblue,urlcolor=darkblue,
	anchorcolor=darkblue,citecolor=darkblue}
\usepackage[colorinlistoftodos,textsize=tiny]{todonotes} 
\usepackage{bbm}
\usepackage{xspace}

\usetikzlibrary{calc}
\newcommand{\Comments}{0}
\newcommand{\mynote}[2]{\ifnum\Comments=1\textcolor{#1}{#2}\fi}
\newcommand{\mytodo}[2]{\ifnum\Comments=1
	\todo[linecolor=#1!80!black,backgroundcolor=#1,bordercolor=#1!80!black]{#2}\fi}

\newcommand{\journal}[1]{}
\ifnum\Comments=1               
\fi

\newcommand{\reals}{\mathbb{R}}

\newcommand{\prop}[1]{\mathrm{prop}[#1]}

\newcommand{\faces}{\mathrm{faces}}

\newcommand{\simplex}{\Delta_\Y}

\newcommand{\A}{\mathcal{A}}
\newcommand{\B}{\mathcal{B}}
\newcommand{\C}{\mathcal{C}}

\newcommand{\E}{\mathbb{E}}
\newcommand{\F}{\mathcal{F}}

\newcommand{\R}{\mathcal{R}}
\newcommand{\Sc}{\mathcal{S}}
\newcommand{\U}{\mathcal{U}}
\newcommand{\V}{\mathcal{V}}
\newcommand{\X}{\mathcal{X}}
\newcommand{\Y}{\mathcal{Y}}

\newcommand{\inprod}[2]{\langle #1, #2 \rangle}
\newcommand{\relint}{\mathrm{relint}}

\newcommand{\toto}{\rightrightarrows}

\newcommand{\conv}{\mathrm{conv}\,}
\newcommand{\cone}{\mathrm{cone}\,}

\newcommand{\ones}{\mathbbm{1}}

\newcommand{\sign}{\mathrm{sign}}
\newcommand{\signstar}{\mathrm{sign}^*}
\renewcommand{\emptyset}{\varnothing}
\newcommand{\onespi}[2]{\ones_{#1,#2}}
\renewcommand{\boxed}[1]{\overline{#1}}
\newcommand{\Vfaces}{\V^{\text{face}}}

\newcommand{\ellabs}{\ell_{\text{abs}}^f}
\newcommand{\fzo}{{f_{\text{0-1}}}}

\newcommand{\Spi}[1]{\{\pi_1,\ldots,\pi_{#1}\}}

\newcommand{\ignore}[1]{}

\DeclareMathOperator*{\argmax}{arg\,max}
\DeclareMathOperator*{\argmin}{arg\,min}

\newtheorem{lemmac}[theorem]{Lemma}
\newtheorem{propositionc}[theorem]{Proposition}

\title[Structured Abstain and Lov\'asz Hinge]{The Structured Abstain Problem and the Lov\'asz Hinge}

\coltauthor{\Name{Jessie Finocchiaro} \Email{jefi8453@colorado.edu}\\
	\Name{Rafael Frongillo} \Email{raf@colorado.edu}\\
	\Name{Enrique Nueve} \Email{ennu6440@colorado.edu}\\
	\addr {University of Colorado Boulder}}

\begin{document}

\maketitle
\begin{abstract}
  The Lov\'asz hinge is a convex surrogate recently proposed for structured binary classification, in which $k$ binary predictions are made simultaneously and the error is judged by a submodular set function.
  Despite its wide usage in image segmentation and related problems, its consistency has remained open.
  We resolve this open question, showing that the Lov\'asz hinge is inconsistent for its desired target unless the set function is modular.
  Leveraging a recent embedding framework, we instead derive the target loss for which the Lov\'asz hinge is consistent.
  This target, which we call the structured abstain problem, allows one to abstain on any subset of the $k$ predictions.
  We derive two link functions, each of which are consistent for all submodular set functions simultaneously.
\end{abstract}

\section{Introduction}

Structured prediction addresses a wide variety of machine learning tasks in which the error of several related predictions is best measured jointly, according to some underlying structure of the problem, rather than independently~\citep{osokin2017structured,gao2011consistency,hazan2010direct,tsochantaridis2005large}.
This structure could be spatial (e.g., images and video), sequential (e.g., text), combinatorial (e.g., subgraphs), or a combination of the above.
As traditional target losses such as 0-1 loss measure error independently, more complex target losses are often introduced to capture the joint structure of these problems.

As with most classification-like settings, optimizing a given discrete target loss is typically intractable.
We therefore seek surrogate losses which are both convex, and thus efficient to optimize, and statistically consistent, meaning they actually solve the desired problem.
Another important factor in structured prediction is that the number of possible labels and/or target predictions is often exponentially large.
For example, in the structured binary classification problem, one makes $k$ simultaneous binary predictions, yielding $2^k$ possible labels.
In these settings, it is crucial to find a surrogate whose prediction space is low-dimensional relative to the relevant parameters.

In general, however, we lack surrogates satisfying all three desiderata: convex, consistent, and low-dimensional~\citep{mcallester2007generalization,nowozin2014optimal}
One promising low-dimensional surrogate for structured binary classification, the Lov\'asz hinge, achieves convexity via the well-known Lov\'asz extension for submodular set functions~\citep{yu2018lovasz,yu2015lovaszarxiv}.
Despite the fact that this surrogate and its generalizations~\citep{berman2018lovasz} have been widely used, e.g.\ in image segmentation and processing~\citep{athar2020stem,chen2020afod,neven2019instance}, its consistency has thus far not been established.

Using the embeddings framework of \citet{finocchiaro2022embedding}, we show the inconsistency of Lov\'asz hinge for structured binary classification (\S~\ref{sec:inconsistency}).
Our proof relies on first determining what the Lov\'asz hinge is actually consistent for: the \emph{structured abstain problem}, a variation of structured binary prediction in which one may abstain on a subset of the predictions (\S~\ref{sec:our-embedding}).
For reasons similar to classification with an abstain option~\citep{ramaswamy2018consistent,bartlett2008classification}, this problem may be of interest to the structured prediction community.
Finally, while the embedding framework shows that a calibrated link must exist, in our case actually deriving such a link is nontrivial.
In \S~\ref{sec:constructing-link} we derive two complementary link functions, both of which are calibrated simultaneously for all submodular set functions parameterizing the problem.

\section{Background}
\label{sec:background}

\subsection{Notation}

See Tables~\ref{tab:notation} and~\ref{tab:notation-proofs} in \S~\ref{app:notation} for full tables of notation.
Throughout, we consider predictions over $k$ binary events, yielding $n = 2^k$ total outcomes, with each label $y \in \Y = \{-1,1\}^k$.
Predictions are generically denoted $r\in\R$; we often take $\R=\Y$, or consider predictions $v \in \V := \{-1,0,1\}^k$ or $u \in \reals^k$.
Loss functions measure these predictions against the observed label $y\in\Y$.
In general, we denote a discrete loss $\ell : \R \times \Y \to \reals_+$ and surrogate $L : \reals^k \times \Y \to \reals_+$.
We also occasionally restrict a loss $L$ to a domain $\Sc \subseteq \R$ and define $L|_{\Sc} : (u,y) \mapsto L(u,y)$ for all $u \in \Sc$.

Let $[k] := \{1,\ldots,k\}$.
When translating from vector functions to set functions, it is often useful to use the shorthand $\{ u \leq c \} := \{ i \in [k] \mid u_i \leq c\}$ for $u\in\reals^k$, $c\in\reals$, and similarly for other set comprehensions.
Additionally, for any $S\subseteq [k]$, we let $\ones_S \in \{0,1\}^k$ with $(\ones_S)_i = 1 \iff i\in S$ be the 0-1 indicator for $S$.
Let $\Sc_k$ denote the set of permutations of $[k]$.
For any permutation $\pi \in \Sc_k$, and any $i\in\{0,1,\ldots,k\}$, define $\onespi{\pi}{i} = \ones_{\Spi{i}}$, where $\onespi{\pi}{0} = 0 \in \reals^k$.

For $u,u'\in\reals^k$, the Hadamard (element-wise) product $u\odot u'\in\reals^k$ given by $(u \odot u')_i = u_iu'_i$ plays a prominent role.
We extend $\odot$ to sets in the natural way; e.g., for $U\subseteq \reals^k$ and $u'\in\reals^k$, we define $U \odot u' = \{u\odot u' \mid u\in U\}$.

We often decompose elements of $u \in \reals^k$ by their sign and absolute value.
To this end, we define $\sign : \reals^k \to \V$ to be the (element-wise) sign of $u$, and use the function $\signstar : \reals^k \to \Y$ to denote an arbitrary function that agrees with $\sign$ when $|u_i|\neq 0$ and break ties arbitrarily at $0$.
We let $|u| \in \reals^k_+$ be the element-wise absolute value $|u|_i = |u_i|$, and frequently use the fact that $|u| = u\odot\signstar(u) = u\odot\sign(u)$.
We define $\boxed u = \sign(u) \odot \min(|u|, \ones)$ to ``clip'' $u$ to $[-1,1]^k$.
Finally, we denote $((u)_+)_i = \max(u_i,0)$.

\subsection{Submodular functions and the Lov\'asz extension}

A set function $f:2^{[k]}\to\reals$ is \emph{submodular} if for all $S,T\subseteq [k]$ we have $f(S) + f(T) \geq f(S\cup T) + f(S\cap T)$.
If this inequality is strict whenever $S$ and $T$ are incomparable, meaning $S\not\subseteq T$ and $T\not\subseteq S$, then we say $f$ is \emph{strictly submodular}.
A function is \emph{modular} if the submodular inequality holds with equality for all $S,T\subseteq [k]$.
The function $f$ is \emph{increasing} if we have $f(S\cup T) \geq f(S)$ for all disjoint $S,T\subseteq [k]$, and \emph{strictly increasing} if the inequality is strict whenever $T\neq\emptyset$.
Finally, we say $f$ is \emph{normalized} if $f(\emptyset) = 0$.
Let $\F_k$ be the class of set functions $f:2^{[k]}\to\reals$ which are submodular, increasing, and normalized.

The structured binary classification problem is given by the following discrete loss $\ell^f:\R\times\Y\to\reals$, with $\R=\Y$,
\begin{equation}
\label{eq:discrete-set-loss}
\ell^f(r,y) = f(\{ r\odot y < 0 \}) = f(\{ i \in [k] \mid r_i \neq y_i \})~.
\end{equation}
In words, $\ell^f$ measures the joint error of the $k$ predictions by applying $f$ to the set of mispredictions, i.e., indices corresponding to incorrect predictions.
For the majority of the paper, we will consider $f\in\F_k$.
In particular, we will make the natural assumption that $f$ is increasing: making an additional error cannot decrease error.
The assumption that $f$ be normalized is without loss of generality.

A classic object related to submodular functions is the \emph{Lov\'asz extension} to $\reals^k$~\citep{lovasz1983submodular}, which is known to be convex when (and only when) $f$ is submodular~\cite[Proposition 3.6]{bach2013learning}.
For any permutation $\pi\in\Sc_k$, define $P_\pi = \{x\in\reals^k_+ \mid x_{\pi_1} \geq \cdots \geq x_{\pi_k}\}$, the set of nonnegative vectors ordered by $\pi$.
The Lov\'asz extension of a normalized set function $f:2^{[k]}\to\reals$ can be formulated in several equivalent ways \citep[Definition 3.1]{bach2013learning}.
\begin{equation}\label{eq:lovasz-ext}
  F(x) = \max_{\pi\in \Sc_k} \sum_{i=1}^k x_{\pi_i} (f(\Spi{i})-f(\Spi{i-1}))~.
\end{equation}
Given any $x\in\reals^k_+$, the argmax in eq.~\eqref{eq:lovasz-ext} is the set $\{\pi\in\Sc_k \mid x \in P_\pi\}$, i.e., the set of all permutations that order the elements of $x$.
For any $\pi \in \Sc_k$ such that $x\in P_\pi$, we may therefore write
\begin{equation}
\label{eq:lovasz-ext-pi-u}
F(x) = \sum_{i=1}^k x_{\pi_i} (f(\Spi{i})-f(\Spi{i-1}))~.
\end{equation}

For any $f\in\F_k$, let $F$ be the Lov\'asz extension of $f$.
\citet{yu2018lovasz} define the \emph{Lov\'asz hinge} as the loss $L^f:\reals^k\times\Y\to\reals_+$ given as follows.
\begin{equation}
\label{eq:lovasz-hinge}
L^f(u,y) = F\bigl((\ones - u \odot y)_+\bigr)~.~
\end{equation}
The Lov\'asz hinge is proposed as a surrogate for the structured binary classification problem in eq.~\eqref{eq:discrete-set-loss}, using the link $\signstar$ to map surrogate predictions $u \in \reals^k$ back to the discrete report space $\R = \Y$.
From eq.~\eqref{eq:lovasz-ext}, the Lov\'asz extension is polyhedral (piecewise-linear and convex) as a maximum of a finite number of affine functions.
Hence $L^f$ is a polyhedral loss function.

Immediately from the definition, the fact that $\odot$ is symmetric, and $x \mapsto x\odot y$ is an involution for any $y\in\Y$, we have the following.
\begin{lemma}\label{lem:lovasz-symmetry}
  For all $u\in\reals^k$ and $y,y'\in\Y$, $L^f(u,y) = L^f(u\odot y',y\odot y')$.
\end{lemma}

\subsection{Running examples}
\label{sec:running-examples}

We will routinely refer to two running examples.
For the first, consider the case where $f$ is modular.
Modular set functions can be parameterized by any $w\in\reals^k_+$, so that $f_w(S) = \sum_{i\in S} w_i$.
In this case $\ell^f$ reduces to \emph{weighted Hamming loss}, and $L^f$ to weighted hinge, the consistency of which is known~\cite[Theorem 15]{gao2011consistency}.
\begin{align}
  L^{f_w}(u,y)
  &= \max_{\pi\in \Sc_k} \sum_{i=1}^k ((1-u\odot y)_+)_{\pi_i} (f(\Spi{i})-f(\Spi{i-1}))
    \nonumber
  \\
  &= \sum_{i=1}^k (1-u_i y_i)_+ (w_i)~.
  \label{eq:weighted-hamming}
\end{align}
For the other example, $\fzo$ given by $\fzo(\emptyset)=0$ and $\fzo(S)=1$ for $S\neq\emptyset$.
Here the Lov\'asz hinge reduces to
\begin{align}
  L^{\fzo}(u,y)
  &= \max_{\pi\in \Sc_k} \sum_{i=1}^k ((1-u\odot y)_+)_{\pi_i} (f(\Spi{i})-f(\Spi{i-1}))
    \nonumber
  \\
  &= \max_{i\in[k]} \; (1-u_i y_i)_+~.
  \label{eq:BEP-simple-form}
\end{align}

In fact, $L^\fzo$ is equivalent to the BEP surrogate by \citet{ramaswamy2018consistent} for the problem of multiclass classification with an abstain option.
The target loss for this problem is $\ell_{1/2}:[n]\cup \{\bot \} \times [n] \to \reals_{+}$ defined by $\ell_{1/2}(r,y)=0$ if $r=y$, $1/2$ if $r=\bot$, and $1$ otherwise.
Here, the report $\bot$ corresponds to ``abstaining'' if no label is sufficiently likely, specifically if no $y\in \Y$  has $p_y \geq 1/2$.
The BEP surrogate is given by
\begin{align}
  \label{eq:BEP}
L_\frac{1}{2}(u,\hat y) &= \left(\max_{j\in [k]}\; B(\hat y)_j u_j +1\right)_+
\end{align}
where $B:[n]\rightarrow \{-1,1\}^k$ is an arbitrary injection.
Substituting $y=-B(\hat y)$ in eq.~\eqref{eq:BEP}, and moving the $(\cdot)_+$ inside, we recover eq.~\eqref{eq:BEP-simple-form}.

\subsection{Property elicitation and calibration}
When considering polyhedral (piecewise-linear and convex) losses, like the Lov\'asz hinge in eq.~\eqref{eq:lovasz-hinge}, \citet[Theorem 8]{finocchiaro2022embedding} show that indirect property elicitation is equivalent to statistical consistency, hence we often use property elicitation as a tool to study consistent polyhedral surrogates for a given discrete loss.

\begin{definition}\label{def:prop-elicits}
	A property $\Gamma : \simplex \to 2^{\R} \setminus \{\emptyset\}$ is a function mapping distributions over labels to reports.
	A loss $L : \R \times \Y \to \reals_+$ \emph{elicits} a property $\Gamma$ if, for all $p \in \simplex$,
	\begin{align*}
	\Gamma(p) = \argmax_{r \in \R} L(r;p)~.~
	\end{align*}
	Moreover, if $\E_{Y \sim p}L(\cdot,Y)$ attains its infimum for all $p \in \simplex$, we say $L$ is \emph{minimizable}, and elicits some unique property, denoted $\prop{L}$.
\end{definition}

In order to connect property elicitation to statistical consistency, we work through the notion of calibration, which is equivalent to consistency in our setting~\citep{bartlett2006convexity,zhang2004statistical,ramaswamy2016convex}.
One desirable characteristic of calibration over consistency is the ability to abstract features $x \in \X$ so that we can simply study the expected loss over labels through the distribution $p \in \simplex$.
We often denote $\E_{Y \sim p} L(u,Y) := L(u;p)$, and $\E_{Y \sim p}\ell(r,Y) := \ell(r;p)$, which more readily aligns with property elicitation definitions.

\begin{definition}\label{def:calibration}
	Let $\ell:\R\times\Y\to\reals$ with $|\R|<\infty$.
	A surrogate $L : \reals^d \times \Y \to \reals_+$ and link $\psi : \reals^d \to \R$ pair $(L, \psi)$ is \emph{calibrated} with respect to $\ell$ if for all $p\in\simplex$,
	\begin{align*}
	\inf_{u : \psi(u) \not \in \prop{\ell}(p)} L(u;p) > \inf_{u \in \reals^d} L(u;p)~.
	\end{align*}
\end{definition}

We use calibration as an equivalent property to study statistical consistency in this paper, and when restricting to polyhedral losses, (indirect) property elicitation implies calibration.
Since the Lov\'asz hinge is a polyhedral surrogate, we specifically use embeddings, which is a special case of property elicitation, to study (in)consistency of this surrogate for structured binary classification.

\subsection{The embedding framework}

We will lean heavily on the embedding framework of~\citet{finocchiaro2019embedding,finocchiaro2022embedding}.
Given a discrete target loss, an embedding maps target reports into $\reals^k$, and observes a surrogate loss $L: \reals^k \times \Y \to \reals_+$ which behaves the same as the target on the embedded points.
The authors show that every polyhedral surrogate embeds some discrete loss, and show embedding implies consistency.
To define embeddings, we first need a notion of representative sets, which allows one to ignore some target reports that are in some sense redundant.
\begin{definition}\label{def:representative-set}
	We say $\Sc \subseteq \R$ is \emph{representative} with respect to the loss $L$ if we have $\argmin_u L(u;p) \cap \Sc \neq \emptyset$ for all $p\in \simplex$.
\end{definition}

\begin{definition}[Embedding]\label{def:loss-embed}
	The loss $L:\reals^d \times \Y \to\reals_+$ \emph{embeds} a loss $\ell:\R \times \Y\to\reals_+$ if there exists a representative set $\Sc$ for $\ell$ and an injective embedding $\varphi:\Sc\to\reals^d$ such that
	(i) for all $r\in\Sc$ and $y \in \Y$ we have $L(\varphi(r),y) = \ell(r,y)$, and (ii) for all $p\in\simplex,r\in\Sc$ we have
	\begin{equation}\label{eq:embed-loss}
	r \in \prop{\ell}(p) \iff \varphi(r) \in \prop{L}(p)~.
	\end{equation}
\end{definition}

Embeddings are intimately tied to polyhedral losses as they have finite representative sets.
Every discrete loss is embedded by some polyhedral loss.
A central tool of this work, however, is the converse: every polyhedral loss embeds some discrete target loss: namely, its restriction to a finite representative set.

\begin{theorem}[{\citep[Thm.\ 3, Prop.\ 1]{finocchiaro2022embedding}}]\label{thm:poly-embeds-discrete}
  A loss $L$ with a finite representative set $\Sc$ embeds $L|_\Sc$.
  Moreover, every polyhedral $L$ has a finite representative set.
\end{theorem}

A central contribution of the embedding framework is to simplify proofs of consistency.
In particular, if a surrogate $L:\reals^k \times \Y \to \reals_+$ embeds a discrete target $\ell:\R\times \Y\to\reals_+$, then there exists a calibrated link function $\psi:\reals^k\to\R$ such that $(L,\psi)$ is consistent with respect to $\ell$.
The proof is constructive, via the notion of \emph{separated} link functions, a fact we will make use of in \S~\ref{sec:constructing-link}; specifically, see Theorem~\ref{thm:embed-implies-consistent}.

\section{Lov\'asz hinge embeds the structured abstain problem}
\label{sec:our-embedding}

As the Lov\'asz hinge is a polyhedral surrogate, Theorem~\ref{thm:poly-embeds-discrete} states that it embeds some discrete loss, which may or may not be the same as the intended target $\ell^f$.
As we saw in \S~\ref{sec:running-examples}, one special case, $L^\fzo$, reduces to the BEP surrogate for multiclass classification with an abstain option, which implies that $L^f$ cannot embed $\ell^f$ in general.
In particular, whatever $L^f$ embeds, it must allow the algorithm to abstain in some sense.
We formalize this intuition by showing $L^f$ embeds the discrete loss $\ellabs$, a variant of structured binary classification which allows abstention on any subset of the $k$ labels.
See \S~\ref{app:omitted-proofs} for all omitted proofs.

\subsection{The filled hypercube is representative}

As a first step, we show that the filled hypercube $R := [-1,1]^k$ is representative for $L^f$, and use this fact to later find a \emph{finite} representative set for $L^f$ and apply Theorem~\ref{thm:poly-embeds-discrete}.
In fact, we show the following stronger statement: surrogate reports outside the filled hypercube $[-1,1]^k$ are dominated on each outcome.

\ignore{\begin{lemmac}\label{lem:lovasz-hypercube-dominates}\end{lemmac}}
\begin{restatable}{lemmac}{lovaszhypercubedominates}
  \label{lem:lovasz-hypercube-dominates}
  For any $u\in\reals^k$, we have $L^f(\boxed u,y) \leq L^f(u,y)$ for all $y\in\Y$.
\end{restatable}
\noindent
Using this result, we may now simplify the Lov\'asz hinge.
When $u\in[-1,1]^k$, we simply have
\begin{equation}\label{eq:lovasz-hinge-simplified}
  L^f|_{R}(u,y) = F(\ones - u \odot y)~,
\end{equation}
as $\ones - u \odot y$ is nonnegative.

\subsection{Affine decomposition of $L^f$}\label{sec:affine-decomposition}

We now give an affine decomposition of $L^f$ on $[-1,1]^k$, which we use throughout.
Recall that for any $\pi\in\Sc_k$ we define $P_\pi = \{x\in\reals^k_+ \mid x_{\pi_1} \geq \cdots \geq x_{\pi_k}\}$.
Letting $V_{\pi} = \{\onespi{\pi}{i} \mid i \in \{0,\ldots,k\}\} \subset \V$,
we have $P_\pi = \cone V_\pi$, the conic hull of $V_\pi$, meaning every $x \in P_\pi$ can be written as a conic combination of elements of $V_\pi$.
For all $i\in\{0,\ldots,k\}$, define the coefficients $\alpha_i : \reals^k_+ \to \reals$ as follows.
For any $x\in\reals^k_+$, define $\alpha_0(x) = 1-x_{[1]} \in \reals$, $\alpha_k(x) = x_{[k]} \geq 0$, and $\alpha_i(x) = x_{[i]} - x_{[i+1]} \geq 0$ for $i\in\{1,\ldots,k-1\}$.
Then
\begin{align}
x
&= \sum_{i=1}^{k} \alpha_i(x) \, \onespi{\pi}{i}
= \sum_{i=0}^{k} \alpha_i(x) \, \onespi{\pi}{i}
~,
\label{eq:alpha-combination}
\end{align}
where we recall that $\ones_{\pi,0} = \vec 0 \in\reals^k$.
We have $\alpha_i(x) \geq 0$ for all $i \in \{1,\ldots,k\}$, so the first equality gives the conic combination.
In the case $x_{[1]} \leq 1$, we have $\alpha_i(x) \geq 0$ for all $i\in\{0,\ldots,k\}$.
Since $\sum_{i=0}^k \alpha_i(x) = 1$, in that case the latter equality in eq.~\eqref{eq:alpha-combination} is a convex combination.
This yields $P_\pi \cap [0,1]^k = \conv V_\pi$.

It is clear from eq.~\eqref{eq:lovasz-ext-pi-u} that $F$ is affine on $P_\pi$ for each $\pi\in\Sc_k$.
We now identify the regions within $[-1,1]^k$ where $L^f(\cdot,y)$ is affine simultaneously for all outcomes $y\in\Y$, using these polyhedra and symmetry in $y$.

Motivated by the above, for any $y\in\Y$ and $\pi\in\Sc_k$, define
\begin{align}
	\label{eq:V-pi}
  V_{\pi,y} &= V_\pi \odot y = \{\onespi{\pi}{i} \odot y \mid i \in \{0,\ldots,k\}\} \subset \V~,
  \\
  \label{eq:P-pi}
  P_{\pi,y} &= \conv (V_{\pi,y}) = \conv (V_\pi) \odot y \subset [-1,1]^k~.
\end{align}
Since $V_{\pi,y}$ is a set of affinely independent vectors, each $P_{\pi,y}$ is a simplex.
Observe that for the case $y=\ones$, we have $P_{\pi,\ones} = P_\pi \cap [0,1]^k$.
Indeed, the other $P_{\pi,y}$ sets are simply reflections of $P_{\pi,\ones}$, as we may write $P_{\pi,y} = P_{\pi,\ones} \odot y$.
We now show that these regions union to the filled hypercube $[-1,1]^k$, and $L^f(\cdot,y)$ is affine on $P_{\pi,y}$ for each $y \in \Y$.

\ignore{\begin{lemmac}\label{lem:p-pi-y}\end{lemmac}}
\begin{restatable}{lemmac}{ppiy}\label{lem:p-pi-y}
	The sets $P_{\pi,y}$ satisfy the following.
	\begin{enumerate}
		\item[(i)]  $\cup_{y\in\Y,\pi\in\Sc_k} P_{\pi,y} = [-1,1]^k$.
		\item[(ii)] For all $f \in \F_k$, $y,y' \in \Y$, and $\pi\in \Sc_k$, the function $L^f(\cdot,y')$ is affine on $P_{\pi,y}$.
	\end{enumerate}
\end{restatable}

\subsection{Embedding the structured abstain problem}

Leveraging the affine decomposition given above, we will now show that the finite set $\V = \{-1,0,1\}^k$ must be representative for $L^f$.
By Theorem~\ref{thm:poly-embeds-discrete}, it will then follow that $L^f$ embeds $\ellabs := L^f|_\V$.
As we describe below, we call $\ellabs$ the \emph{structured abstain problem} because the predictions $v\in\V$ allow one to ``abstain'' on an index $i$ by setting $v_i = 0$.
\ignore{\begin{lemmac}\label{lem:UcapR-union-faces}\end{lemmac}}
\begin{restatable}{lemmac}{UcapRunionfaces}\label{lem:UcapR-union-faces}
	Given a polyhedral loss function $L : \reals^k \times \Y \to \reals_+$, let $\C$ be a collection of polyhedral subsets of $\reals^k$ such that for all $y\in\Y$, $L(\cdot,y)$ is affine on each $C_i\in\C$, and denote $\faces(C_i)$ as the set of faces of $C_i$.
	Let $R = \cup \C$ be the union of these polyhedral subsets.
	Then for all $p\in\simplex$, $\prop{L}(p)\cap R = \cup \F$ for some $\F \subseteq \cup_{i} \faces(C_i)$.
\end{restatable}

\begin{proposition}\label{prop:lovasz-V-rep}
  The set $\V=\{-1,0,1 \}^k$ is representative for $L^f$.
\end{proposition}
\begin{proof}
    Let $\C=\{P_{\pi,y} |\forall \; \pi \in \Sc_k,y\in\Y\}$ and $R=\cup \C=\cup_{\pi\in\Sc_k ,y\in\Y}P_{\pi ,y}=[-1,1]^k$ by Lemma~\ref{lem:p-pi-y}(i).
    Since every $P_{\pi ,y}$ is affine w.r.t $L^f$ according to Lemma~\ref{lem:p-pi-y}(ii), we have by Lemma~\ref{lem:UcapR-union-faces} $\forall p\in\simplex$, $\prop{L}(p)\cap R = \cup F$ where $\F \subseteq \cup_{\pi,y} \faces(P_{\pi,y})$. Yet, by the construction of $P_{\pi,y}$, every face contains some number of vertices from $\V$. Therefore, $\forall$ $p\in\simplex$, $\prop{L}(p)\cap \V\neq \emptyset$ which by definition means that $\V$ is representative for $L^f$.
\end{proof}

\begin{theorem}
	\label{thm:lovasz-embeds}
	The Lov\'asz hinge $L^f$ embeds $\ellabs:\V\times\Y\to\reals_+$ given by
	\begin{equation}
    \label{eq:lovasz-embeds}
    \ellabs(v,y) = f(\{ v\odot y < 0 \}) + f(\{ v\odot y \leq 0 \})~.
	\end{equation}
\end{theorem}
\begin{proof}
  From Proposition~\ref{prop:lovasz-V-rep} and Theorem~\ref{thm:poly-embeds-discrete}, $L^f$ embeds $L^f|_\V$.
  It therefore remains only to establish the set-theoretic form of $L^f|_\V$ as the loss $\ellabs$ in eq.~\eqref{eq:lovasz-embeds}.

  Let $v\in\V,y\in\Y$ be given.
  We may write
  \begin{align*}
    \ones - v\odot y = 0\cdot\ones_{\{ v\odot y > 0 \}} + 1\cdot\ones_{\{ v\odot y = 0 \}} + 2\cdot\ones_{\{ v\odot y < 0 \}}~.
  \end{align*}
	Now combining eq.~\eqref{eq:lovasz-hinge-simplified} and \citet[Prop 3.1(h)]{bach2013learning}, we may therefore write
  \begin{align*}
    L^f(v,y)
    &= F(\ones - v\odot y)
    \\
    &= (2-1) f(\{ v\odot y < 0 \}) + (1-0) f(\{ v\odot y < 0 \}\cup\{ v\odot y = 0 \}) + 0 f([k])
    \\
    &= f(\{ v\odot y < 0 \}) + f(\{ v\odot y \leq 0 \})~,
  \end{align*}
  as was to be shown.
\end{proof}

We can interpret $\ellabs$ as a structured abstain problem, where the algorithm is allowed to abstain on a given prediction by giving a zero instead of $\pm 1$.
Specifically, we can say the algorithm abstains on the set of indices $A_v = \{ v = 0 \}$.

To make this interpretation more clear, let $r = \signstar(v)$, which is forced to choose a label $\pm 1$ for each zero prediction.
The corresponding set of mispredictions for fixed $y \in \Y$ would be $M^y = \{ r \odot y < 0 \}$.
We can rewrite eq.~\eqref{eq:lovasz-embeds} in terms of these sets as $\ellabs(v,y) = f(M^y\setminus A_v) + f(M^y\cup A_v)$.
Contrasting with $\ellabs(r,y) = 2 f(\{r \odot y < 0\}) = f(M^y) + f(M^y)$, the abstain option allows one to reduce loss in the first term at the expense of a sure loss in the second term.
Intuitively, when there is large uncertainty about the labels of a set of indices $A \subseteq [k]$, by submodularity the algorithm would prefer to abstain on $A$ than take a chance on predicting.

When relating to submodularity, we will often find it useful to rewrite the misprediction set $M^y$ above in terms of two sets of labels: $S_v = \{ \signstar(v) > 0 \}$ and $S_y = \{ y > 0 \}$.
Then $M^y = S_v \triangle S_y$, and thus
\begin{equation}
  \label{eq:lovasz-embeds-symdiff}
  \ellabs(v,y) = f(S_v\triangle S_y\setminus A_v) + f(S_v\triangle S_y\cup A_v)~,
\end{equation}
where $\triangle$ is the symmetric difference operator $S\triangle T := (S\setminus T) \cup (T\setminus S)$.
To avoid additional parentheses, throughout we assume $\triangle$ has operator precedence over $\setminus$, $\cap$, and $\cup$.

For $r\in\Y$, we have $\ellabs(r,y) = 2\ell^f(r,y)$, meaning $\ellabs$ matches (twice) $\ell^f$ on $\Y$.
Were the ``abstain'' reports $v \in \V \setminus \Y$ dominated, then we would indeed have consistency.
Following the above intuition, however, we can show that whenever $f$ is submodular but not modular, there are situations where abstaining is uniquely optimal (relative to $\V$), leading to inconsistency.

\section{Inconsistency for structured binary classification}
\label{sec:inconsistency}

Leveraging the embedded loss $\ellabs$, we now show that $L^f$ is inconsistent for its intended target $\ell^f$, except when $f$ is modular.
As the modular case is already well understood, under the name \emph{weighted Hamming loss} (\S~\ref{sec:running-examples}), this result essentially says that $L^f$ is inconsistent for all nontrivial cases.

As $L^f$ embeds $\ellabs$, to show inconsistency we may focus on reports $v \in \V \setminus \Y$, i.e., those that abstain on at least one index.
Intuitively, if such a report is ever optimal, then $L^f$ with the link $\signstar$ has a ``blind spot'' with respect to the indices in $A_v := \{ v = 0 \}$.
We can leverage this blind spot to ``fool'' $L^f$, by making it link to an incorrect report.
In particular, we will focus on the uniform distribution $\bar p$ on $\Y$, and perturb it slightly to find an optimal $L^f$ point $v\in\V$ which maps to a $\ell^f$ suboptimal report $\signstar(v)$.
In fact, we will show that one can always find such a point violating consistency, unless $f$ is modular.

Given our focus on the uniform distribution, the following definition will be useful: for any set function $f$, let $\bar f := 2^{-k} \sum_{S\subseteq [k]} f(S) \in\reals$.
The next two lemmas relate $\bar f$ and $f([k])$ to expected loss and modularity.
The proofs follow from summing the submodularity inequality over all possible subsets, and observing that at least one of them is strict when $f$ is non-modular.

\ignore{\begin{lemmac}\label{lem:2-bar-f}\end{lemmac}}
\begin{restatable}{lemmac}{twobarf}
	\label{lem:2-bar-f}
	For all $v \in \V$, $\ellabs(v;\bar p) \geq f([k])$.
 	For all $r \in \Y$, $\ellabs(r;\bar p) = 2\bar f$.
\end{restatable}
\vspace*{-15pt}
\ignore{\begin{lemmac}\label{lem:bar-f}\end{lemmac}}
\begin{restatable}{lemmac}{barf}
	\label{lem:bar-f}
	Let $f$ be submodular and normalized.
	Then $\bar f \geq f([k])/2$, and $\bar f = f([k])/2$ if and only if $f$ is modular.
\end{restatable}

Typical proofs of inconsistency identify a particular pair of distributions $p,p'\in\Delta(\Y)$ for which the same surrogate report $u$ is optimal, yet two distinct target reports are uniquely optimal for each, $r$ for $p$ and $r'$ for $p'$.
As $u$ cannot link to both $r$ and $r'$, one concludes that the surrogate cannot be consistent.
We follow this same general approach, but face one additional hurdle: we wish to show inconsistency of $L^f$ for \emph{all} non-modular $f$ simultaneously.
In particular, the distributions $p,p'$ may need to depend on the choice $f$, so at first glance it may seem that such an argument would be quite complex.
We achieve a relatively straightforward analysis by defining $p,p'$ based on only a single parameter of $f$; the optimal surrogate report itself may be entirely governed by $f$, but will lead to inconsistency regardless.

The proof relies on a similar symmetry observation as Lemma~\ref{lem:lovasz-symmetry}, that $L^f(u\odot y',y\odot y') = L^f(u,y)$; in particular, $\prop{L^f}$ has the same symmetry.
For $p\in\Delta(\Y)$ and $r\in\Y$, define $p\odot r \in \Delta(\Y)$ by $(p\odot r)_y = p_{y\odot r}$.
\ignore{\begin{lemmac}\label{lem:lovasz-property-symmetry}\end{lemmac}}
\begin{restatable}{lemmac}{lovaszpropertysymmetry}\label{lem:lovasz-property-symmetry}
  For all $p\in\Delta(\Y)$ and $r\in\Y$, $\prop{L^f}(p\odot r) = \prop{L^f}(p)\odot r$.
\end{restatable}

\begin{theorem}
	Let $f$ be submodular, normalized, and increasing.
	Then $(L^f,\sign)$ is consistent if and only if $f$ is modular.
\end{theorem}
\begin{proof}
  When $f$ is modular, we may write $f=f_w$ for some $w\in\reals_+^k$.
  Here $L^{f_w}$ is weighted hinge loss (eq.~\eqref{eq:weighted-hamming}), which is known to be consistent for $\ell^{f_w}$, which is weighted Hamming loss~\citep[Theorem 15]{gao2011consistency}.
  (Briefly, for all $p\in\simplex$ the loss $L^{f_w}(\cdot;p)$ is linear on $[-1,1]^k$, so it is minimized at a vertex $r\in\Y$.
    Hence $\Y$ is representative, so Theorem~\ref{thm:poly-embeds-discrete} gives that $L^{f_w}$ embeds $L^{f_w}|_\Y = 2\ell^{f_w}$.
	Consistency follows from Theorem~\ref{thm:embed-implies-consistent}.)

	Now suppose $f$ is submodular but not modular.
	As $f$ is increasing, we will assume without loss of generality that $f(\{i\}) > 0$ for all $i\in [k]$, which is equivalent to $f(S) > 0$ for all $S\neq\emptyset$; otherwise, $f(T) = f(T\setminus\{i\})$ for all $T\subseteq [k]$, so discard $i$ from $[k]$ and continue.
	In particular, we have $\{\emptyset\} = \argmin_{S\subseteq [k]} f(S)$.

	Define $\epsilon = \bar f / (2\bar f - f([k]))$.
  We have $\epsilon > 0$ by Lemma~\ref{lem:bar-f} and submodularity of $f$.
	For any $y\in\Y$, let $p^y = (1-\epsilon) \bar p + \epsilon \delta_y$, where again $\bar p$ is the uniform distribution, and $\delta_y$ is the point distribution on $y$.

  First, for all $r\in\Y$ with $r\neq y$, we have $\{r\odot y < 0\} \neq \emptyset = \{y\odot y < 0\}$.
  Since $\{\emptyset\} = \argmin_{S\subseteq [k]} f(S)$, we have
  \begin{align*}
    \ell^f(r;p^y)
    &= (1-\epsilon) 2 \bar f + \epsilon \, 2f(\{r\odot y < 0\})\\
    &> (1-\epsilon)2 \bar f + \epsilon \, 2f(\{y\odot y < 0\})\\
    &=\ell^f(y;p^y)~,
  \end{align*}
  giving $\prop{\ell^f}(p^y) = \{y\}$.
  On the other hand, from Lemma~\ref{lem:2-bar-f} and the fact that $\ellabs$ agrees with $\ell^f$, we have for all $r\in\Y$,
	\begin{align*}
    \ellabs(r;p^y)
    \geq \ellabs(y;p^y)
    =(1-\epsilon)2 \bar f
    > f([k]) = \ellabs(0;p^y)~.
	\end{align*}
	We conclude there exists some optimal report $v \in \prop{\ellabs}(p^y) \setminus \Y$.
  By Theorem~\ref{thm:lovasz-embeds}, $v\in\prop{L^f}(p^y)$ as well.

  As $v\notin\Y$, in particular, $\{v = 0\} \neq \emptyset$.
  Now define $y'\in\Y$ to disagree with $y$ on $\{v = 0\}$; formally,
$y'_i = v_i$ if $v_i \neq 0$ and $y'_i = -y_i$ if $v_i = 0$.
  Although $y'\neq y$ (as $\{v=0\} \neq \emptyset$), we have by construction that $v \odot (y \odot y') = v$.
  Furthermore, $p^y \odot (y\odot y') = p^{y'}$.
  By Theorem~\ref{thm:lovasz-embeds} and Lemma~\ref{lem:lovasz-property-symmetry} then, $v \in \prop{L^f}(p^{y'})$.
  By the above, however, we also have $\{y'\} = \prop{\ell^f}(p^{y'})$.
  As $\signstar(v)$ cannot be both $y$ and $y'$, at least one of $p^y$ and $p^{y'}$ exhibits the inconsistency of $L^f$ for $\ell^f$.
  Specifically, calibration is violated (Definition~\ref{def:calibration}) as $v$ achieves the optimal $L^f$-loss for both $p^y$ and $p^{y'}$, but for at least one, links to a report not in $\prop{\ell^f}$.
\end{proof}

\section{Constructing a calibrated link for $\ellabs$}
\label{sec:constructing-link}

As $L^f$ embeds $\ellabs$ from Theorem~\ref{thm:lovasz-embeds}, Theorem~\ref{thm:embed-implies-consistent} below further implies $L^f$ is consistent with respect to $\ellabs$ for some link function.
Yet, the design of such a link function is not immediately clear.
Indeed, natural choices turn out to be inconsistent in general, such as the threshold link $\psi_c$ for $c > 0$ used by the BEP surrogate (\S~\ref{sec:running-examples}), which given by $(\psi_c(u))_i=0$ whenever $|u_i|<c$ and $(\psi_c(u))_i = \sign(u_i)$ otherwise (Figure~\ref{fig:k2-links}).
We instead follow the construction of an $\epsilon$-separated link from~\citet{finocchiaro2022embedding}, resulting in two consistent link functions.
Interestingly, while these links do not depend on $f$, they are calibrated with respect to $\ellabs$ for all $f \in \F_k$ simultaneously.
See \S~\ref{app:omitted-proofs} for omitted proofs.

\subsection{Approach via separated link functions}
\label{sec:separ-link-funct}

For any polyhedral loss $L$ which embeds a target discrete loss $\ell$, \citet{finocchiaro2022embedding} give a construction of a link function $\psi$ such that $(L,\psi)$ is calibrated with respect to $\ell$.
Their construction is based on $\epsilon$-separation, as follows.

\begin{definition}[{\citep[Construction 1]{finocchiaro2022embedding}}] \label{def:eps-thick-link}
	Let a polyhedral loss $L:\reals^d\times\Y\to\reals_+$ that embeds some discrete loss $\ell:\R\times\Y\to\reals_+$ be given, along with $\epsilon > 0$, and a norm $\|\cdot\|$.
  The \emph{$\epsilon$-thickened link envelope} $\Psi:\reals^d\toto\R$ is constructed as follows.
	Define $\U = \{\prop{L}(p) : p \in \simplex\}$ and, for each $U \in \U$, let $R_U = \{r \in \R : \varphi(r) \in U\}$, the reports whose embedding points are in $U$.
	Initialize by setting $\Psi(u) = \R$ for all $u\in\reals^d$.
	Then for each $U \in \U$, and all points $u$ such that $\inf_{u^* \in U} \|u^*-u\| < \epsilon$, update $\Psi(u) = \Psi(u) \cap R_U$.
\end{definition}

We say a link envelope $\Psi$ is nonempty pointwise if $\Psi(u) \neq \emptyset$ for all $u\in\reals^d$.
Similarly, a link function $\psi$ is pointwise contained in $\Psi$ if $\psi(u) \in \Psi(u)$ for all $u\in\reals^d$.
\ignore{\begin{theorem}\label{thm:embed-implies-consistent}\end{theorem}}
\begin{restatable}[{\citep[Theorems 5, 6]{finocchiaro2022embedding}}]{theoremc}{embedimpliesconsistent}\label{thm:embed-implies-consistent}
  Let $L:\reals^k \times \Y \to \reals_+$ embed a discrete target $\ell:\R\times \Y\to\reals_+$, and let $\Psi$ be defined as in Definition~\ref{def:eps-thick-link}.
  Then $\Psi$ is nonempty pointwise for all sufficiently small $\epsilon$.
  Furthermore, for any link function $\psi$ pointwise contained in $\Psi$, the pair $(L,\psi)$ is consistent with respect to $\ell$.
\end{restatable}

Essentially, 
this construction ``thickens'' each of potentially optimal set and ensures surrogate report that is close to these regions must be linked to a representative report contained in that set.
One can consider $\Psi$ the resulting ``link envelope'', from which a calibrated link may be arbitrarily chosen pointwise.

To apply this construction to the Lov\'asz hinge $L^f$, let $\Psi^f$ be the envelope $\Psi$ from Definition~\ref{def:eps-thick-link} applied to $L^f$.
We immediately encounter a complication: 
as the link envelope $\Psi^f$ depends on the choice of $f$, it is entirely possible that no single link function is contained in the envelopes $\Psi^f$ for all $f\in\F_k$, i.e., is simultaneously calibrated for $L^f$ for \emph{all} such $f$.
If no simultaneous link existed, the construction and analysis would have to be tailored carefully to each $f\in\F_k$.
Interestingly, we show that such a simultaneous link does exist.

To find a link which is calibrated for all $f$, we identify certain structure which is common to Lov\'asz hinges $L^f$.
We encode this structure in a common link envelope $\hat \Psi$, and then show in Proposition~\ref{prop:psi-containment} that, for all $f \in \F_k$ and $u \in \reals^k$, we have $\hat \Psi(u) \subseteq \Psi^f(u)$.
We then show that $\hat \Psi$ is nonempty for sufficiently small $\epsilon$, meaning it contains a link option pointwise.
This link is therefore contained in all the link envelopes $\Psi^f$ for all $f$, and hence is calibrated with respect to $\ellabs$ for all $f \in \F_k$ simultaneously.

\begin{figure}
\centering
	\begin{minipage}[b]{0.3\textwidth}
	\centering
	\includegraphics[width=\textwidth]{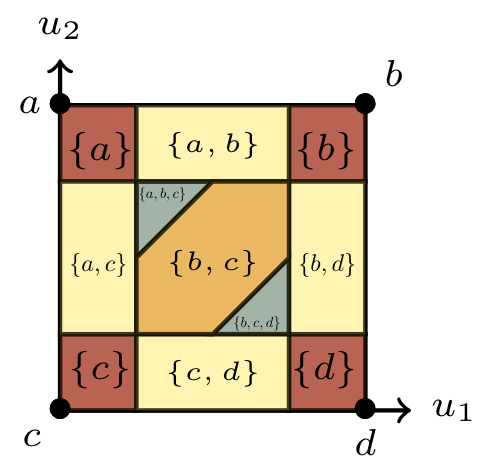}
\end{minipage}
\hfill
  \begin{minipage}[b]{0.3\textwidth}
    \includegraphics[width=\textwidth]{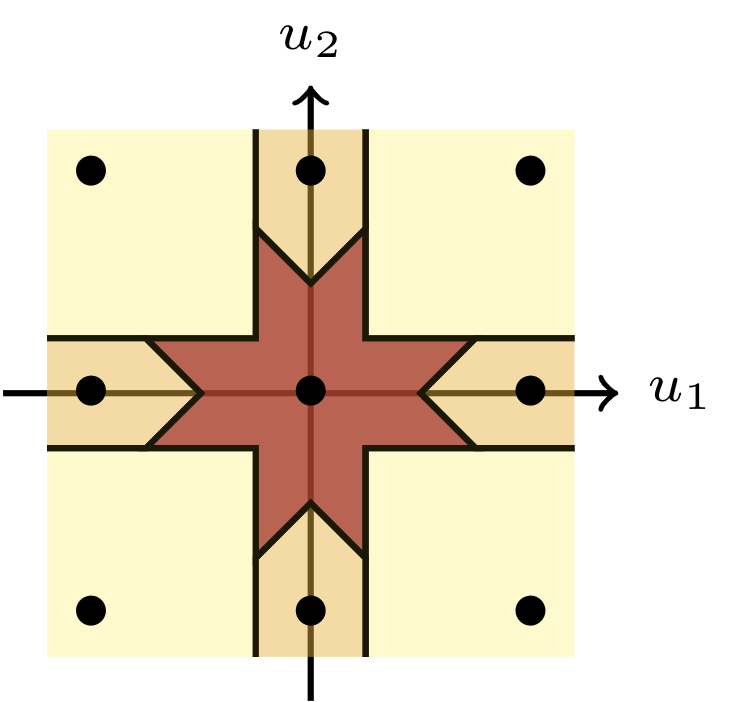}
  \end{minipage}
\hfill
  \begin{minipage}[b]{0.3\textwidth}
    \includegraphics[width=\textwidth]{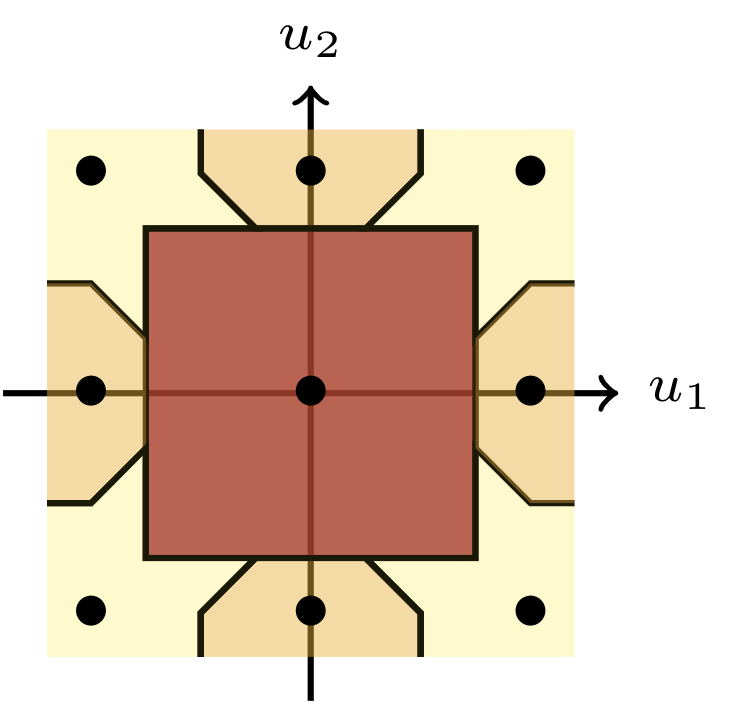}
  \end{minipage}
  \caption{The link envelope $\hat\Psi$ (left) and link functions $\psi^*_\epsilon$ (middle) and $\psi^\diamond_\epsilon$ (right) for $k=2$ and $\epsilon=\frac{1}{4}$.
    The envelope $\hat \Psi$ is pictured for $u \in \reals_+^2$, with each region labeled by the value of $\hat\Psi$; a link is calibrated if it always links to one of the nodes in the region.
    The values for the link functions $\psi^*_\epsilon$ and $\psi^\diamond_\epsilon$ are given by the unique point $v \in \V$ that each depicted region contains.
    In particular, both links satisfy the constraints from $\hat\Psi$ (left) and thus are calibrated.
    Interestingly, $\psi^*_\epsilon$ matches the link proposed for the BEP surrogate~\cite{ramaswamy2018consistent} if one additionally sends all abstain reports (with any 0 coordinate) to $\bot$.}
\label{fig:k2-links}
\end{figure}

\subsection{The common link envelope $\hat \Psi$}\label{sec:calibrated-via-separation}
We now present our link envelope $\hat\Psi$, used to construct calibrated links (Figure~\ref{fig:k2-links}, left).

\begin{definition}\label{def:Psi-candidates}
  Let $\Vfaces := \cup_{\pi \in \Sc_k, y \in \Y} 2^{V_{\pi,y}}$ be the subsets of $\V$ whose convex hulls are faces of some $P_{\pi,y}$ polytope.
	Define $\hat \Psi:\reals^k \to 2^\V$ by $\hat\Psi(u) = \cap \{ V \in \Vfaces \mid d_\infty (\conv V, \boxed u) < \epsilon\}$.
\end{definition}

Now we show that $\hat \Psi \subseteq \Psi^f$ pointwise.
The proof uses the fact that both $\hat \Psi(u)$ and $\Psi^f(u)$ are constructed by the intersections of sets, and shows that the sets generating $\hat \Psi(u)$ are subsets of those generating $\Psi^f(u)$ for all $f \in \F_k$.
In particular, every possible optimal set in the range of $\prop{L^f}$ is a union of faces generated by convex hulls of elements of $\Vfaces$.

\ignore{\begin{propositionc}\label{prop:psi-containment}\end{propositionc}}
\begin{restatable}{propositionc}{psicontainment}\label{prop:psi-containment}
  For all $f \in \F_k$ and $u\in\reals^k$, we have $\hat\Psi(u) \subseteq \Psi^f(u)$.
\end{restatable}

We now characterize the link envelope $\hat \Psi(u)$ in terms of the coordinates of $u$.
In particular, $\hat \Psi(u)$ consists of the embedding points $v\in\V$ that make up the intersection of the faces from $\Vfaces$ that are $\epsilon$ close to $u$.
We can express these points in terms of the ordered elements of $|u|$.
In particular, such a point $v\in\V$ appears in the intersection exactly when the corresponding elements of $|u|$ are $2\epsilon$-far from each other, since otherwise we can find a face not containing $v$ which is $\epsilon$-close to $u$ (Proposition~\ref{prop:hat-psi-char}).
Therefore, $\Psi$ is always nonempty when $\epsilon$ is small enough to guarantee a gap of at least $2\epsilon$ in the ordered elements of $|u|$ (Lemma~\ref{lem:hat-Psi-nonempty}).

\begin{restatable}{propositionc}{hatpsichar}\label{prop:hat-psi-char}
  Let $u\in\reals^k$, and let $\pi\in\Sc_k$ order the elements of $|u|$ (descending).
  For the purposes of the following, define $|u_{\pi_0}| = 1+\epsilon$ and $|u_{\pi_{k+1}}| = -\epsilon$.
  Then we have
  \begin{equation}\label{eq:hat-psi-char}
    \hat\Psi(u) = \{\onespi{\pi}{i} \odot \signstar(u) \mid i\in\{0,1,\ldots,k\}, \; |u_{\pi_i}| \geq |u_{\pi_{i+1}}| + 2\epsilon \}
  \end{equation}
\end{restatable}
\begin{restatable}{lemmac}{hatPsinonempty}\label{lem:hat-Psi-nonempty}
  $\hat\Psi$ is nonempty pointwise
  if and only if $\epsilon \in (0, \tfrac 1 {2k}]$.
\end{restatable}

\subsection{Two calibrated link functions from $\hat \Psi$}

We now proceed to construct two $\epsilon$-separated links, $\psi^*_\epsilon$, which abstains as little as possible, and $\psi^\diamond_\epsilon$, which abstains as much as possible.
For sufficiently small $\epsilon$, both links are pointwise contained in $\hat\Psi$, giving calibration from Theorem~\ref{thm:embed-implies-consistent}.

\begin{definition}\label{def:link-max-istar}
  Let $\epsilon>0$ be fixed.
  Let $u\in\reals^k$, and let $\pi\in\Sc_k$ order the elements of $|u|$.
  Given any $u\in\reals^k$, let $i^*\in\{0,\ldots,k\}$ be the largest index $i$ such that $|u_{\pi_i}| - |u_{\pi_{i+1}}| \geq 2\epsilon$ where we define $|u_{\pi_0}| = 1+\epsilon$ and $|u_{\pi_{k+1}}| = -\epsilon$.
  Then define
  \begin{equation}
    \label{eq:our-link}
    \psi^*_\epsilon(u) = \ones_{\pi,i^*} \odot \signstar(u)~.
  \end{equation}
  Similarly, let $i^\diamond \in\{0,\ldots,k\}$ be the smallest index $i$ such that $|u_{\pi_i}| - |u_{\pi_{i+1}}| \geq 2\epsilon$ and define
    \begin{equation}
    \label{eq:our-link-abs}
    \psi^\diamond_\epsilon(u) = \ones_{\pi,i^\diamond} \odot \signstar(u)~.
  \end{equation}
\end{definition}

\begin{theorem}\label{thm:well-defined-cal}
  Let $\epsilon \in (0, 1/2k]$, and fix any $f\in\F_k$.
  Then
  $(L^f,\psi^*_\epsilon)$ and
  $(L^f,\psi^\diamond_\epsilon)$ are well-defined and calibrated with respect to $\ellabs$.
\end{theorem}

\begin{proof}
  Lemma~\ref{lem:hat-Psi-nonempty} shows that the indices $i^*$ and $i^\diamond$ in Definition~\ref{def:link-max-istar} always exist when
  $\epsilon \in (0,\frac{1}{2k}]$,
  which shows that $\psi^*_\epsilon$ and $\psi^*_\diamond$ are well-defined.
  By construction, we have $\psi_{\epsilon}^* (u)\in \hat \Psi (u)$ and $\psi_{\epsilon}^\diamond (u)\in \hat \Psi (u)$ for all $u\in\reals^k$.
	As Proposition~\ref{prop:psi-containment} states that $\hat \Psi \subseteq \Psi^f$ pointwise, we then have $\psi^*_\epsilon, \psi^\diamond_\epsilon \in \Psi^f$ pointwise.
  Finally, Theorem~\ref{thm:embed-implies-consistent} states that any link function contained in $\Psi^f$ pointwise is calibrated.
\end{proof}

The two proposed link functions, $\psi^*_\epsilon$ and $\psi^\diamond_\epsilon$, differ by how often one abstains vs the other.
The first, $\psi^*_\epsilon$, has a smaller abstain region which decreases in volume as $\epsilon$ decreases.
Meanwhile, $\psi^\diamond_\epsilon$ has a larger abstain region which increases in volume as $\epsilon$ decreases.
Based on one's preferred risk, either $\psi^*_\epsilon$ if risk seeking otherwise  $\psi^\diamond_\epsilon$ if risk adverse could be used.
The difference between how often $\psi^*_\epsilon$ and $\psi^\diamond_\epsilon$ abstain is demonstrated for $k=2$ in Figure~\ref{fig:k2-links}.

\newpage

\section*{Acknowledgements}
We would like to thank Eric Balkanski for help with a lemma about submodular functions

\bibliographystyle{plainnat}
\bibliography{refs}

\newpage
\appendix
\section{Notation tables} \label{app:notation}
See Tables~\ref{tab:notation} and \ref{tab:notation-proofs}.

\begin{table}
	\centering
	\begin{tabular}{@{\extracolsep{4pt}}ll}
		Notation & Explanation \\
		\toprule
		$k$ & Number of binary events\\
    $[k] := \{1,\ldots,k\}$ & Index set\\
		$y \in \Y = \{-1,1\}^k$  & Label space    \\
		$v \in \V = \{-1,0,1\}^k$  & (Abstain) prediction space    \\
		$r \in \R$  & General prediction space \\
		$R = [-1,1]^k$ & The filled $\pm 1$ hypercube\\
		$u\in \reals^k$  & Surrogate prediction space    \\
		$\{u \leq c\} = \{i \in  [k] \mid u_i \leq c\}$ & Set of indices of $u$ less than $c$ \\
		$(u \odot u')_i = u_i u'_i$  & Hadamard (element-wise) product \\
		$U \odot u' = \{u \odot u' \mid u \in U\}$  & Hadamard product on a set $U \subseteq \reals^k$\\
		$\sign : \reals^k \to \V$  & Sign function including $0$\\
		$\signstar : \reals^k \to \Y$ & Sign function breaking ties arbitrarily at $0$\\
		$|u| \in \reals^k_+$ s.t. $|u|_i = |u_i|$ & Observe $|u| = u \odot \signstar(u) = u \odot \sign(u)$ \\
		$\boxed u = \sign(u) \odot \min(|u|, \ones)$ & ``Clipping'' of $u$ to $R$ \\
		$\ones_S \in \{0,1\}^k$ s.t. $(\ones_S)_i = 1 \iff i \in S$ & $0-1$ Indicator on set $S \subseteq [k]$\\
		$\pi \in \Sc_k$ & Permutations of $[k]$\\
		$f \in \F_k$ & Set of normalized, increasing, and submodular \\
		& set functions $f : 2^k \to \reals_+$.\\
		$\ell^f(r,y) = f(\{r \odot y < 0\})$ & Structured binary classification eq.~\eqref{eq:discrete-set-loss}\\
		$F(x)$ 
             & Lova\'sz extension for $x \in \reals_+^k$ in eq.~\eqref{eq:lovasz-ext}\\
		$L^f(u,y) = F((\ones - u \odot y)_+)$ & Lov\'asz hinge eq.~\eqref{eq:lovasz-hinge}\\
		$\ellabs(v,y) = f(\{v \odot y < 0\})+ f(\{v \odot y \leq 0\})$ & Structured abstain problem eq.~\eqref{eq:lovasz-embeds}\\
		\bottomrule
	\end{tabular}
\caption{Table of general notation}\label{tab:notation}
\end{table}

\begin{table}
	\centering
	\begin{tabular}{@{\extracolsep{4pt}}ll}
		Notation & Explanation \\
		\toprule
		$\ones_{\pi,i} = \ones_{\{\pi_1, \ldots, \pi_i\}}$ with $\ones_{\pi,0} = \vec 0$ & Indicator of first $i$ elements of $\pi$\\
	$V_{\pi} = \{ \ones_{\pi,i} \mid i \in \{0, \ldots, k\}\}$ & Elements of $\V$ ordered by $\pi$\\
		$V_{\pi,y} = V_\pi \odot y$ & Signed elements of $\V$  ordered by $\pi$.\\
		$P_\pi = \{ x \in \reals^k_+ \mid x_{\pi_1} \geq \ldots \geq x_{\pi_k}\}$ & elements of $\reals_+^k$ ordered by $\pi$\\
		$P_{\pi,y} = \conv V_\pi \odot y$ & Elements of $P_\pi$ signed by $y$\\
		 $\Vfaces = \cup_{\pi \in \Sc_k, y \in \Y} 2^{V_{\pi,y}}$ & Subsets of $\V$ whose convex hulls are \\
		 & faces of some $P_{\pi,y}$ polytope.\\
		$\hat \Psi(u) = \cap \{V \in \Vfaces \mid d_\infty (\conv V, u) < \epsilon\}$ & Proposed general link envelope.\\
		$\U^f = \prop{L^f}(\simplex)$ & Range of property elicited by Lov\'asz hinge\\
		$\Psi^f(u) = \cap \{U \in \U^f \mid d_\infty (U, u) < \epsilon\} \cap \V$ & Link envelope for given $f \in \F_k$.\\
		\bottomrule
	\end{tabular}
	\caption{Table of notation used for proofs}\label{tab:notation-proofs}
\end{table}

\section{Omitted Proofs}\label{app:omitted-proofs}
\subsection{Omitted Proofs from \S~\ref{sec:our-embedding}}
\lovaszhypercubedominates*
\begin{proof}
	Fix $y\in\Y$.
	Let $w = \ones - u\odot y$ and $\boxed w = \ones - \boxed u\odot y$, so that $L^f(u,y) = F(w_+)$ and $L^f(\boxed u,y) = F(\boxed w_+)$.
	We will first show that $\boxed w_+ = \min(w_+,2)$, where the minimum is element-wise.

	For $i\in[k]$ such that $|u_i| \leq 1$, we have $\boxed u_i = u_i$.
	Thus $(w_+)_i = (1 - u_i y_i)_+ = (1 - u_i y_i)_+ = (\boxed w_+)_i$.
	Furthermore, we have $0 \leq (w_+)_i = (\boxed w_+)_i \leq 2$.
	Now suppose $|u_i| > 1$.
	If $y_iu_i > 0$, i.e., $\sign(u_i) = y_i$, then $1-y_iu_i = 1-|u_i| < 0$, so $(w_+)_i = 0$.
	For $\boxed u$, we similarly have $(\boxed w_+)_i = (1-|\boxed u_i|)_+ = 0$.
	In the other case, $y_iu_i < 0$, so $(w_+)_i = 1+|u_i| > 2$ and $(\boxed w_+)_i = 1+|\boxed u_i| = 2$.
	Therefore, we have $\boxed w_+ = \min(w_+, 2)$.

	Now, let $\pi\in\Sc_k$ be a permutation that orders the elements of $w_+$.
	Observe that $\pi$ orders the elements of $\boxed w_+$ as well, since the vectors are identical except for values above 2, which are all mapped to 2.
	By eq.~\eqref{eq:lovasz-ext-pi-u}, we thus have
	\begin{align*}
	F(w_+) - F(\boxed w_+)
	&= \sum_{i=1}^k (w_+)_{\pi_i} (f(\Spi{i})-f(\Spi{i-1}))
	\\
	& \qquad     -  \sum_{i=1}^k (\boxed w_+)_{\pi_i} (f(\Spi{i})-f(\Spi{i-1}))
	\\
	&= \sum_{i=1}^k (w_+-\boxed w_+)_{\pi_i} (f(\Spi{i})-f(\Spi{i-1}))
	\\
	& \geq 0~,
	\end{align*}
	where we have used the fact that $f$ is increasing and $\boxed w_+ \leq w_+$ element-wise.
	As $y$ was arbitrary, this holds for all $y \in \Y$.
\end{proof}

\ppiy*
\begin{proof}
	For (i), take any $u\in[-1,1]^k$.
	Letting $y = \signstar(u)$, we have $u \odot y = |u| \in \reals^k_+$.
	Taking $\pi$ to be any permutation ordering the elements of $u \odot y$, we have
	$u\odot y \in P_{\pi} \cap \reals^k_+$.
	Notice, since $u\odot y \in P_{\pi} \cap \reals^k_+$ and $u\in[-1,1]^k$, we additionally have $u \odot y = |u| \in P_\pi \cap [0,1]^k$.
	Since $\onespi{\pi}{i}$ for $i \in \{0,\ldots,k\}$ form $V_{\pi}$ and $P_{\pi}$ is the convex hull of points in $V_\pi$, showing there is an $\alpha$ such that $u = \sum_i \alpha_i \onespi{\pi}{i}$ suffices to conclude $u \in P_{\pi,y}$.
	We can write $u \odot y$ as the convex combination $u\odot y = \sum_{i=0}^k \alpha_i(u\odot y) \onespi{\pi}{i}$, as in eq.~\eqref{eq:alpha-combination}.
	Thus $u = u\odot y\odot y = \sum_{i=0}^k \alpha_i(u\odot y) \onespi{\pi}{i} \odot y$, so $u\in P_{\pi,y}$.
	Therefore,  every $u\in[-1,1]^k$ is in some $P_{\pi,y}$, we have $\cup_{y\in\Y,\pi\in\Sc_k} P_{\pi,y} \supseteq [-1,1]^k$.
	Moreover, every $P_{\pi,y}\subseteq [-1,1]^k$ by construction, and equality follows.

	For (ii), first observe for all $\pi\in\Sc_k$, the function $F$ is affine on $P_\pi$, immediately from eq.~\eqref{eq:lovasz-ext-pi-u}.
	To show $L(\cdot, y') = F((\ones - u\odot y')_+)$ is affine on $P_{\pi,y}$ for all $y,y'\in\Y, \pi\in\Sc_k$, it therefore suffices to show there exists some $\pi'$ such that $\{\ones - u \odot y' \mid u \in P_{\pi,y}\} \subseteq P_{\pi'}$.
	We construct $\pi'$, unraveling the permutation $\pi$ into two permutations, depending on the sign of $y\odot y'$.
	Recall from the discussion following eq.~\eqref{eq:lovasz-ext} that $\pi$ orders the elements of $u\odot y = |u|$ in decreasing order.
	Observe that $u \odot y' = u\odot (y\odot y)\odot y' = (u\odot y) \odot (y\odot y') = |u| \odot (y\odot y')$.
	Thus, $\pi$ orders the elements of $u\odot y'$ in decreasing order among indices $i$ with $y_i y'_i > 0$, and increasing order on the others.
	Therefore $\pi$ orders the elements of $\ones - u\odot y'$ in increasing order among indices $i$ with $y_i y'_i > 0$, and decreasing order on the others.
	Taking $\pi'$ to be the order given by sorting the elements in $\{y\odot y' < 0\}$ according to $\pi$, followed by the remaining elements according to the reverse of $\pi$, we have shown $\ones - u\odot y' \in P_{\pi'}$.
\end{proof}

We now introduce a lemma used in the proof of Lemma~\ref{lem:UcapR-union-faces}.
\begin{lemmac}\label{lem:subgradients-on-relint-affine}
	Let $L : \reals^k \to \reals_+$ be polyhedral function that is affine on the polyhedron $C$.
	For any $x \in \relint(C)$ and any $z \in C$, we have $\partial L(x) \subseteq \partial L(z)$.
\end{lemmac}
\begin{proof}
  Fix $x\in\relint(C)$.
  Since $L$ is affine on $C$, then there exists some $w' \in \reals^k, b\in\reals$ such that $L(z) = \inprod {w'} z + b$ for all $z \in C$.
  Thus, we have $L(z) - L(x) = (\inprod{w'}{z} + b) - (\inprod{w'}{x} + b) = \inprod{w'}{z-x}$ for all $z\in C$.

We claim that for all $w \in \partial L(x)$, and all $z \in C$, we have $\inprod{w}{z-x} = \inprod{w'}{z-x}$.
To prove this claim, observe that
\begin{equation}
  \label{eq:blahblah}
  \inprod{w'}{z-x} = L(z) - L(x) \geq \inprod{w}{z-x} \text{ for all } z\in C~,
\end{equation}
by the subgradient inequality and affineness of $L$ on $C$.
Assume for a contradiction that $\inprod{w'}{z-x} > \inprod{w}{z-x}$ for some $z \in C$.
Since $x \in \relint(C)$, there is an $\epsilon < 0$ such that $z' := x + \epsilon(z-x) \in C$.
Therefore, we have 
\begin{align*}
\inprod{w'}{z'-x} = \inprod{w'}{\epsilon(z-x)} = \epsilon\inprod{w'}{z-x} < \epsilon\inprod{w}{z-x} = \inprod{w}{\epsilon(z-x)} = \inprod{w}{z'-x}~,
\end{align*}
where we use the fact that $\epsilon < 0$ to flip the inequality.
We have now contradicted eq.~\eqref{eq:blahblah} for the point $z'$.

Since we now have $L(z) - L(x) = \inprod{w'}{z-x} = \inprod{w}{z-x}$ for all $z \in C$,
consider $w \in \partial L(x)$.
Then we have, for all $v \in \reals^k$,
\begin{align*}
L(v) - L(z) &= (L(v) - L(x)) + (L(x) - L(z)) & \\
&\geq \inprod{w}{v-x} + \inprod{w}{x-z} & \\
&=\inprod{w}{v-z}~, &
\end{align*}
where the inequality follows from the subgradient inequality and the claim.
Thus $w \in \partial L(z)$, which completes the proof.
\end{proof}

A corollary of Lemma~\ref{lem:subgradients-on-relint-affine} is that subdifferentials are constant on $\relint(C)$ for any face $C$ such that $L$ is affine as the subset inclusion holds in both directions.

\UcapRunionfaces*
\begin{proof}
	Fix $p \in \simplex$.
	For any $u \in R \cap \prop{L}(p)$, there is some $\C' \subseteq \C$ such that $u \in C_j$ for all $C_j\in\C'$.
	For now, let us simply consider any $C_j \in \C'$.
	Observe that $u\in \relint (F_j)$ for exactly one face $F_j$ of $C_j$.

	By convexity of $L$, we have $u \in \prop{L}(p) \iff 0 \in \partial L(u;p)$.
	Moreover, as $u \in \relint(F_j)$, we have $\partial L(u;p)\subseteq \partial L(z;p)$ for all $z \in F_j$ by Lemma~\ref{lem:subgradients-on-relint-affine}.
	Thus, $0 \in \partial L(u;p)$ implies $0 \in \partial L(z;p)$ for all $z \in F_j$.
	Moreover, $0 \in \partial L(z;p)$ for all $z \in F_j$ if and only if $z \in \prop{L}(p)$ for all $z \in F_j$, and thus we have $F_j \subseteq \prop{L}(p)$.

	As the value $u$ and the index $j$ were arbitrary, this holds for all such faces in $G(u) := \cup \{F_j \subseteq C_j \in \C' \mid u \in \relint(F_j)\}$.
	Now, take $\F = \{ G(u) \mid u \in R \cap \prop{L}(p)\}$; hence $\prop{L}(p) \cap R = \cup \F$.
	Moreover, $\F \subseteq \cup_i \faces(C_i)$.
\end{proof}

\subsection{Omitted Proofs for \S~\ref{sec:inconsistency}}

\twobarf*
\begin{proof}
	Let $A_v = \{ v = 0 \}$ and $B_v = [k]\setminus A_v$.
	Recall that $\bar p$ is the uniform distribution on $2^k$ outcomes.
	Then we have
	\begin{align*}
	\ellabs(v;\bar p)
	&= 2^{-k} \sum_{S\subseteq [k]} f(S_v\triangle S\setminus A_v) + f(S_v\triangle S\cup A_v)
	\\
	&= 2^{-|B_v|} \sum_{T\subseteq B_v} f(T) + f(T\cup A_v)
	\\
	&= \frac 1 2 \; 2^{-|B_v|} \sum_{T\subseteq B_v} f(T) + f(B_v\setminus T) + f(T\cup A_v) + f((B_v\setminus T)\cup A_v)
	\\
	&\geq \frac 1 2 \left( f(B_v) + f(\emptyset) + f([k]) + f(A_v) \right)
	\\
	&\geq \frac 1 2 \left( f([k]) + f([k]) \right) = f([k])~,
	\end{align*}
	where we use submodularity in both inequalities.
	The second statement follows from the second equality above after setting $A_v=\emptyset$, as then $B_v = [k]$ and thus $T$ ranges over all of $2^{[k]}$.
\end{proof}

\barf*
\begin{proof}
	The inequality follows from Lemma~\ref{lem:2-bar-f} with $r \in \Y$.
	Next, note that if $f$ is modular we trivially have $\bar f = f([k])/2$.
	If $f$ is submodular but not modular, we must have some $S\subseteq [k]$ and $i\in S$ such that $f(S) - f(S\setminus\{i\}) < f(\{i\})$.
	By submodularity, we conclude that $f([k]) - f([k]\setminus\{i\}) < f(\{i\})$ as well; rearranging, $f(\{i\}) + f([k]\setminus\{i\}) > f([k]) = f([k]) + f(\emptyset)$.
	Again examining the proof of Lemma~\ref{lem:2-bar-f}, we see that the first inequality must be strict, as we have one such $T\subseteq [k]$, namely $T=\{i\}$, for which the inequality in submodularity is strict.
\end{proof}

\lovaszpropertysymmetry*
\begin{proof}
	We define $p\odot r \in \Delta(\Y)$ by $(p\odot r)_y = p_{y\odot r}$.
	\begin{align*}
	\prop{L^f}(p\odot r)
	&= \argmin_{u\in\reals^k} \sum_{y\in\Y} (p\odot r)_y L^f(u,y) \\
	&= \argmin_{u\in\reals^k} \sum_{y\in\Y} p_{y\odot r} L^f(u,y) & \text{Definition of $p\odot r$}\\
	&= \argmin_{u\in\reals^k} \sum_{y\in\Y} p_{y\odot r} L^f(u\odot r,y\odot r) & \text{Lemma~\ref{lem:lovasz-symmetry}}\\
	&= \argmin_{u\in\reals^k} \sum_{y'\in\Y} p_{y'} L^f(u\odot r,y') & \text{Substituting $y=y'\odot r$} \\
	&= \left(\argmin_{u'\in\reals^k} \sum_{y'\in\Y} p_{y'} L^f(u',y')\right)\odot r\\
	&= \prop{L^f}(p)\odot r
	\end{align*}
\end{proof}

\subsection{Omitted proofs from \S~\ref{sec:constructing-link}}

Since $\boxed u \in R$, ``clipping'' $u'$ to $\boxed{u'}$ can only reduce element-wise distance, and therefore $d_\infty(\boxed u, \cdot)$ is still small, which allows us to restrict our attention to $R$.

\ignore{\begin{lemmac}\label{lem:truncated-u-also-close}\end{lemmac}}
\begin{restatable}{lemmac}{truncatedualsoclose}\label{lem:truncated-u-also-close}
	Let $f \in \F_k$.
	For all $U \in \prop{L^f}(\simplex)$, $u\in\reals^k$, and $0<\epsilon<2$, if $d_\infty(U,u) < \epsilon$ then $d_\infty(U\cap [-1,1]^k, \boxed u) < \epsilon$.
\end{restatable}
\begin{proof}
	Since $U$ is closed, we have some closest point $u'\in U$ to $u$, meaning $d_\infty(u',u) = d_\infty(U,u) < \epsilon$.
	As $\boxed u' \in U$ by a corollary of Lemma~\ref{lem:lovasz-hypercube-dominates}, it suffices to show $d_\infty(\boxed u, \boxed u') < \epsilon$.

	For each $i \in [k]$, we consider three cases.
	It suffices to show distance does not increase on each element by the choice of the $d_\infty(\cdot,\cdot)$ distance.

	The cases are as follows: (i) $u_i = \boxed u_i$ and $u' = \boxed u'_i$, (ii) $u_i \neq \boxed u_i$ and $u'_i \neq \boxed u'_i$, and (iii) $u_i = \boxed u_i$ and $u'_i \neq \boxed u'_i$ (WLOG).
	Case (i) is trivial as $|u_i - u'_i| = |\boxed u_i - \boxed u'_i| < \epsilon$.
	In case (ii), we must have $\sign(u)_i = \sign(u')_i$ as $d_\infty(u,u') < \epsilon \implies |u_i - u'_i| < \epsilon$.
	If both $u_i$ and $u'_i$ are outside $[-1,1]^k$, this inequality is only true (for $\epsilon < 2)$ if the sign matches.
	Therefore $|\boxed u_i - \boxed u'_i| = |\sign(u)_i - \sign(u')_i| = 0 < \epsilon$.
	In case (iii), we have $\epsilon > |u_i - u'_i| > |u_i - 1| = |u_i - \boxed u'_i|$.
	As absolute difference in each element does not increase, the $d_\infty(\cdot, \cdot)$ distance does not increase.
\end{proof}

We now proceed to statements about the link envelope construction $\hat \Psi$.

\psicontainment*
\begin{proof}
  Let us define
  \begin{align*}
    \A(u) &:= 	\{V \in \Vfaces \mid d_\infty(\conv V, \boxed u) < \epsilon\}~,
    \\
    \B(u) &:= \{U \cap \V \mid U\in \U^f, \, d_\infty(U, u) < \epsilon\}~,
  \end{align*}
	so that $\hat \Psi(u) = \cap \A(u)$ and $\Psi^f(u) = \cap \B(u)$.
  We wish to show $\cap \A(u) \subseteq \cap \B(u)$.
  It thus suffices to show the following claim: for all $B \in \B(u)$ we have some $A\in \A(u)$ with $A \subseteq B$.
  Since then $v \in \cap \A(u)$ implies $v\in A$ for all $A\in\A(u)$, which by the claim implies $v\in B$ for all $B\in\B(u)$ and thus $v\in\cap\B(u)$.

  Let $B \in \B(u)$, so we may write $B = U\cap \V$ for $U\in \U^f$ with $d_\infty(U, u) < \epsilon$.
  By Lemma~\ref{lem:truncated-u-also-close} we have $d_\infty(U \cap R, \bar u) = d_\infty(U,u) < \epsilon$.
	From Lemma~\ref{lem:p-pi-y}, the set $R = [-1,1]^k = \cup_{\pi \in \Sc_k, y \in \Y} P_{\pi,y}$ is the union of polyhedral subsets of $\reals^k$, and $L(\cdot,y)$ is affine on each $P_{\pi,y}$.
  By Lemma~\ref{lem:UcapR-union-faces}, we then have $U \cap R = \cup \F$ for some $\F \subseteq \cup_{\pi,y}\faces(P_{\pi,y})$.
  As each such face can be written as $\conv V$ for some $V\in\Vfaces$, we have some $\V' \subseteq \Vfaces$ such that $U \cap R = \cup \F = \cup_{V \in \V'}\conv V$.
  Now $\min_{V\in\V'} d_\infty(\conv V, \boxed u) = d_\infty(U \cap R, \boxed u) < \epsilon$, so we have some $V \in \V'$ such that $d_\infty(\conv V, \boxed u) < \epsilon$.
  Thus $V \in \A(u)$ by definition.
  As $\conv V \subseteq U\cap R$, we have $V = (\conv V)\cap \V \subseteq (U\cap R)\cap \V = U\cap \V = B$, which proves the claim.

\end{proof}

\begin{lemmac}\label{lem:smallest-vert-rep}
	Fix $u \in [-1,1]^k$, and consider $ \pi , y$ such that $u \in P_{\pi,y}$.
	Then $V_{\pi , y}^{u} := \{\onespi{\pi}{i} \odot y \mid i \in \{0,\ldots,k\}, \alpha_i(|u|) \neq 0 \}$
	is the
	smallest (in cardinality) set of vertices such that $V_{\pi , y}^{u} \subseteq   V_{\pi,y} $ and $u \in \mathrm{conv}(V_{\pi , y}^{u})$. 
\end{lemmac}
\begin{proof}
	First, observe that $V^u_{\pi,y} \subseteq V_{\pi,y}$ by construction, as the first set is constructed the same as the second, with one additional constraint.
	Moreover, we have $u = \sum_{i=1}^k \alpha_i(|u|) u_i = \sum_{i : \alpha_i(|u|) \neq 0} \alpha_i(|u|) u_i \in \conv V^u_{\pi,y}$.

	Now recall $P_{\pi,y}$ is a simplex (see ``Linear interpolation on simplices'' \citet[pg.\ 167]{bach2013learning}) thus, by properties of simplex, each $u\in P_{\pi,y}$ has a unique convex combination expressed by the vertices of $  V_{\pi,y}$ which are affinely independent~\citep[pg.\ 14, Thm 2.3]{brondsted2012introduction}.
	Therefore, every vertex $i$ with a non-zero weighting $\alpha_i(|u|)\neq0$ is necessary in order to express $u$ as a convex combination due to the affine independence of the vertices.
	Thus, $V_{\pi , y}^{u} := \{\onespi{\pi}{i} \odot y \mid i \in \{0,\ldots,k\}, \alpha_i(|u|) \neq 0 \}$, and as $|V_{\pi , y}^{u}| < \infty$, has to be the smallest (in cardinality) set of vertices such such that $V_{\pi , y}^{u} \subseteq V_{\pi,y}$ and $u \in \mathrm{conv}(V_{\pi , y}^{u})$.
\end{proof}

Moreover, $\hat\Psi$ is symmetric around signed permutations.
\begin{lemmac}\label{lem:hat-Psi-symmetry-y}
	For all $u\in \reals^k$, $y\in \Y$, and $\pi\in \Sc_k$, we have $\hat \Psi (\pi(u\odot y))=\pi(\hat \Psi (u)\odot y)$, where we define $(\pi x)_i = x_{\pi_i}$ and we extend this operation to sets.
\end{lemmac}
\begin{proof}
	The proof that the permutation part $(\hat\Psi(\pi u) = \pi\hat\Psi(u))$ is straightforward from the definition.
	For sign changes, observe $\boxed{u\odot y}=\sign(u\odot y)\dot \min(|u\odot y|,1)=\sign(u)\odot y\odot \min(|u|,1)=\boxed{u}\odot y$.
  The operation $u\mapsto u\odot y$ is an isometry for the infinity norm as a special case of signed permutations, here the identity permutation~\citep[Theorem 2.3]{chang1992certain}.
	For all closed
	 $U\subseteq\reals^k$, we therefore have
	$d_{\infty}(\boxed{u\odot y},U\odot y)=d_{\infty}(\boxed{u}\odot y, U\odot y)=d_{\infty}(\boxed{u},U)$.
	Therefore,
	\begin{align*}
	\hat\Psi(u\odot y)
	&= \cap \{ V \in \Vfaces \mid d_\infty (\conv V, \boxed{u\odot y}) < \epsilon\} &
	\\
	&= \cap \{ V \in \Vfaces \mid d_\infty (\conv V\odot y, \boxed{u}) < \epsilon\} & \quad\text{$\boxed{u \odot y} = \boxed u \odot y$, and $\boxed u \odot y \odot y = \boxed u$} \\[-2pt]
	& & \text{with $d_\infty$ preserved under $\odot$.}  \\
	&= \cap \{ V \in \Vfaces \mid d_\infty (\conv V, \boxed{u}) < \epsilon\}\odot y &
	\\
	&= \hat\Psi(u)\odot y~. &
	\end{align*}
\end{proof}

\hatpsichar*
\begin{proof}
	We will show the statement for $u\in \reals^k_+$ with $u_1 \geq \cdots \geq u_{k}$, i.e., where $u\in P_{\pi^*}$ where $\pi^*$ is the identity permutation.
	Lemma~\ref{lem:hat-Psi-symmetry-y} then gives the result, as we now argue.
	For any $u\in\reals^k$, let $\pi\in\Sc_k$ order the elements of $|u|$, and let $y=\signstar(u)$.
	Then $\pi(u\odot y) = \pi|u| \in P_{\pi^*}$.
	Once we show eq.~\eqref{eq:hat-psi-char} is true on the unsigned, ordered case, eq.~\eqref{eq:hat-psi-char} gives $\hat\Psi(\pi|u|) = \{\onespi{\pi^*}{i} \mid i\in\{0,1,\ldots,k\}, \; |u_{\pi_i}| \geq |u_{\pi_{i+1}}| + 2\epsilon \}$.
	Thus $\hat\Psi(u) = \hat\Psi(\pi(u\odot y)) = \pi(\hat\Psi(u)\odot y) = \{\pi(\onespi{\pi^*}{i}\odot y) \mid i\in\{0,1,\ldots,k\}, \; |u_{\pi_i}| \geq |u_{\pi_{i+1}}| + 2\epsilon \} = \{\onespi{\pi}{i}\odot \signstar(u) \mid i\in\{0,1,\ldots,k\}, \; |u_{\pi_i}| \geq |u_{\pi_{i+1}}| + 2\epsilon \}$.

	To begin, we show that for any $i\in \{ 0 ,1,\dots ,k\}$ where $u_{_i}<u_{i+1}+2\epsilon$, $\onespi{\pi^*}{i}\notin \hat \Psi (u)$ by the contrapositive.
	First, suppose that there exists an $i\in \{0, 1,\dots ,k \}$ such that $u_{i}<u_{i+1}+2\epsilon$.
	Since $u$ is ordered, we know that $0\leq u_{i} - u_{i+1} < 2\epsilon$.

	Let $z=\frac{u_i + u_{i+1}}{2}$ and define $\hat u$ such that $\hat{u}_i=z$ and $\hat{u}_{i+1}=z$ while every other index of $\hat u$ is equal to $u$.
	Observe $u_i -z < \epsilon$ and $z - u_{i+1} <  \epsilon$, and thus $d_{\infty}(u,\hat u) < \epsilon$ as $d_\infty(\cdot,\cdot)$ is measured component-wise.
	By Lemma~\ref{lem:smallest-vert-rep} and construction of $\alpha$ in the first paragraph of \S~\ref{sec:affine-decomposition}, we have $\alpha_{i}(\hat u)=\hat u_{i}-\hat u_{i+1}=0$, we have $\hat u \in \conv (V_i)$, where $V_i := V_{\pi^* , y}^{u}\setminus \{\onespi{\pi^*}{i}\}$.
	Since $\hat u \in \conv (V_i)$ and $d_{\infty}(\hat u,u)<\epsilon$, we have $V_i \supseteq \hat{\Psi}(u)$, and therefore, for any $i\in \{ 0 ,1,\dots , k\}$ such that $u_i<u_{i+1}+2\epsilon$, $\onespi{\pi^*}{i}\notin \hat \Psi (u)$.

	Now, for the converse, fix any $u\in P_{\pi^*}$ with $i \in \{ 0,1,\dots , k\}$ such that $u_i \geq u_{i+1}+2\epsilon$.
	For any $u'\in \reals^k$ such that $d_{\infty}(u,u')<\epsilon$, we claim that $\alpha_i (u') \neq 0$, and therefore $\onespi{\pi^*}{i}\in \hat \Psi(u)$.

	Assume there exists a $u'\in \reals^k$ such that $d_{\infty}(u,u')<\epsilon$ for some $i\in \{ 0,1,\dots , k\}$.
	Given that $d_{\infty}(u,u')<\epsilon$, $u'_{j}\in (u_j-\epsilon,u_j+\epsilon) \, \forall j \in \{0,\ldots, k\}$: namely, for $j = i$ and $i+1$. 
	However, since $u_i-u_{i+1}\geq 2 \epsilon$, $(u_i-\epsilon,u_i+\epsilon)\cap (u_{i+1}-\epsilon,u_{i+1}+\epsilon) = \emptyset$.
	Therefore, $\alpha_i(u') = u'_{i} - u'_{i+1} > 0$.
	By Lemma~\ref{lem:smallest-vert-rep},  we then have $\onespi{\pi^*}{i}\in V^u_{\pi^*,y}$, which is the smallest set $V$ such that $d_\infty(\conv V, u) < \epsilon$, and is therefore in the intersection of all such sets; this intersection yields $\hat \Psi(u)$.
	Thus, we have $\hat\Psi(u) = \{\onespi{\pi}{i} \odot \signstar(u) \mid i\in\{0,1,\ldots,k\}, \; u_i \geq u_{i+1} + 2\epsilon \}$.
\end{proof}

\begin{figure}[!]
	\centering
	\includegraphics[width=0.5\textwidth]{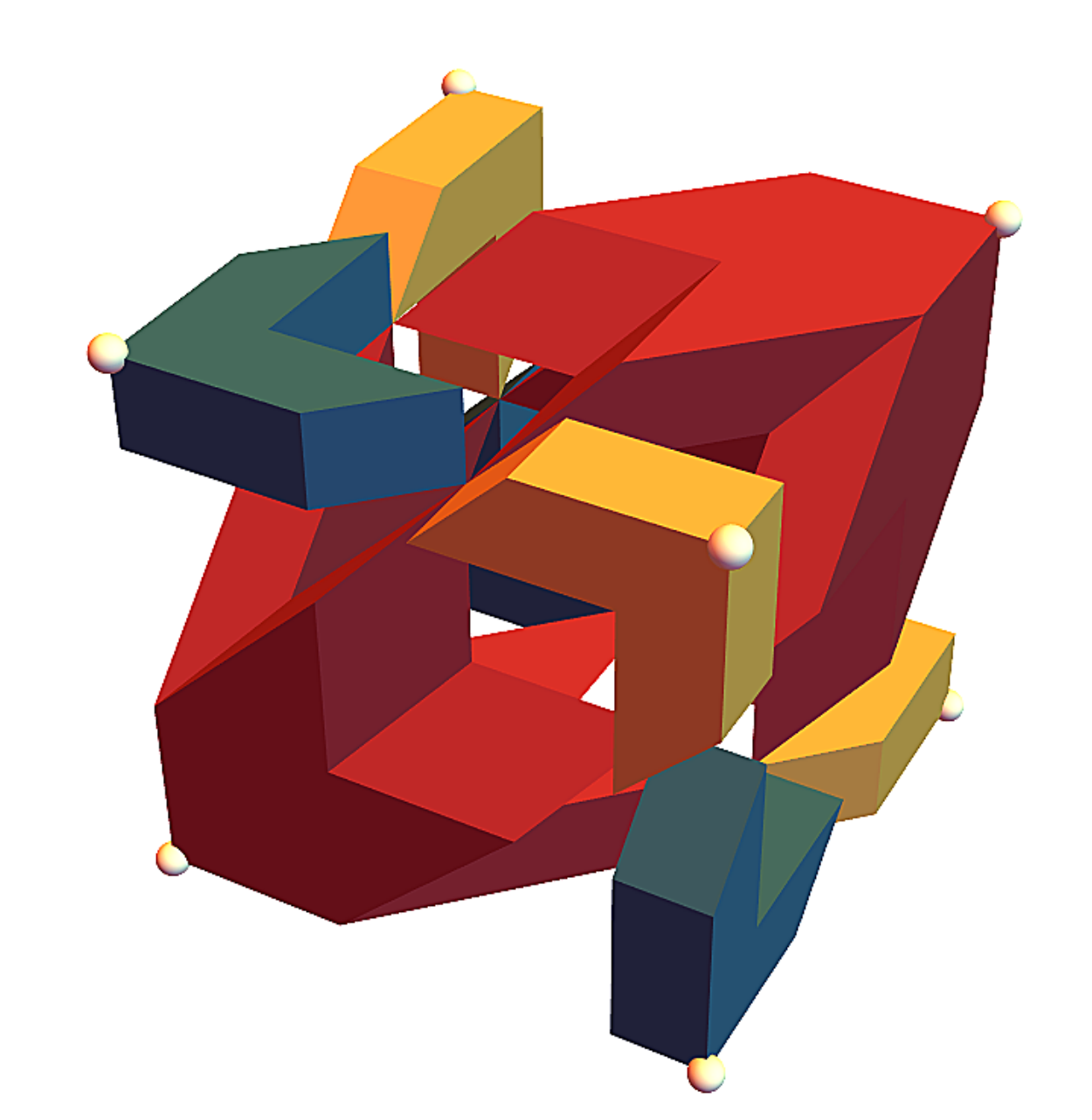}
  \caption{$\hat \Psi(u)$ for $u \in \reals_+^3$ and $\epsilon = \frac{1}{6}$. Each colored region connected to a particular node corresponds to a $v \in \{0,1 \}^3 \subseteq \V$ and at a point $u$, a calibrated link must link to one of the $v$ in the region.}
  		\label{fig:Psi-hat}
\end{figure}

\hatPsinonempty*
\begin{proof}
	By Lemma~\ref{lem:hat-Psi-symmetry-y}, it suffices to show the statement for $u\in\reals^k_+$.
	We will show the contrapositive in both directions: there exists $u\in\reals^k_+$ such that $\hat\Psi(u) = \emptyset$ if and only if $\epsilon > \tfrac 1 {2k}$.

	For any $u\in\reals^k_+$, define $u_{k+1} = -\epsilon$ and $u_{0} = 1+\epsilon$ as in Proposition~\ref{prop:hat-psi-char}.
	From the characterization in Proposition~\ref{prop:hat-psi-char} (eq.~\eqref{eq:hat-psi-char}), we have $\hat\Psi(u)=\emptyset$ if and only if
	\begin{equation}\label{eq:hat-psi-nonempty}
	u_i - u_{i+1} < 2\epsilon \text{ for all } i\in\{0,1,\ldots,k\}~.
	\end{equation}

	We may also write
	\begin{equation}\label{eq:u-telescope-trick}
	1+\epsilon = u_0 = u_{k+1} + \sum_{i=0}^{k} (u_i-u_{i+1}) = \sum_{i=0}^{k} (u_i-u_{i+1}) - \epsilon~.
	\end{equation}
	If there exists $u\in\reals^k_+$ with $\hat \Psi(u) = \emptyset$, then eq.~\eqref{eq:hat-psi-nonempty}~and~\eqref{eq:u-telescope-trick} together imply $1+2\epsilon = \sum_{i=0}^{k} (u_i-u_{i+1}) < (k+1)(2\epsilon)$, giving $\epsilon > \tfrac 1 {2k}$.
	For the converse, if $\epsilon > \tfrac 1 {2k}$, take $u\in\reals^k_+$ given by $u_i = \tfrac {2i-1}{2k}$.
	Then $u_0 - u_1 = 1+\epsilon - (1- \tfrac 1 {2k}) < 2\epsilon$ and $u_k-u_{k+1} = \tfrac 1 {2k} + \epsilon < 2\epsilon$, and for $i\in\{1,\ldots,k-1\}$, we have $u_{i+1}-u_i = \tfrac 1 k < 2\epsilon$, giving eq.~\eqref{eq:hat-psi-nonempty}.
\end{proof}

\end{document}